\documentclass[conference]{IEEEtran}

\usepackage{cite}
\usepackage{booktabs}
\usepackage{amsthm}
\usepackage{newtxmath}
\newtheorem{problem}{problem}
\newtheorem{theorem}{theorem}
\usepackage{multirow}
\usepackage{subfigure}
\usepackage{graphicx}
\usepackage[ruled,vlined]{algorithm2e}

\IEEEoverridecommandlockouts    

\begin{document}

\title{Adversarial Attack on Hierarchical \\ Graph Pooling Neural Networks}

%

\author{Haoteng~Tang,      Guixiang~Ma,     Yurong~Chen,       Lei~Guo,        Wei~Wang,  Bo~Zeng, Liang~Zhan

\thanks{H. Tang, Y. Chen, L. Guo and L. Zhan are with the Department
of Electrical and Computer Engineering, University of Pittsburgh, Pittsburgh,
PA, 15260 USA (e-mail: \{hat64, yuc127, Lei.guo, liang.zhan\}@pitt.edu)}
\thanks{G. Ma is with Intel Labs, Hillsboro, USA (e-mail:guixiang.ma@intel.com). *This work is done before G.Ma joined Intel.}
\thanks{B. Zeng and W. Wang are with the Department
of Industrial Engineering, University of Pittsburgh, Pittsburgh,
PA, 15260 USA (e-mail: \{bzeng, w.wei\}@pitt.edu)}
}

\maketitle

\begin{abstract}
Recent years have witnessed the emergence and development of graph neural networks (GNNs), which have been shown as a powerful approach for graph representation learning in many tasks, such as node classification and graph classification. The research on the robustness of these models has also started to attract attentions in the machine learning field. However, most of the existing work in this area focus on the GNNs for node-level tasks, while little work has been done to study the robustness of the GNNs for the graph classification task. In this paper, we aim to explore the vulnerability of the Hierarchical Graph Pooling (HGP) Neural Networks, which are advanced GNNs that perform very well in the graph classification in terms of prediction accuracy. We propose an adversarial attack framework for this task. Specifically, we design a surrogate model that consists of convolutional and pooling operators to generate adversarial samples to fool the hierarchical GNN-based graph classification models. We set the preserved nodes by the pooling operator as our attack targets, and then we perturb the attack targets slightly to fool the pooling operator in hierarchical GNNs so that they will select the wrong nodes to preserve. We show the adversarial samples generated from multiple datasets by our surrogate model have enough transferability to attack current state-of-art graph classification models. Furthermore, we conduct the robust train on the target models and demonstrate that the retrained graph classification models are able to better defend against the attack from the adversarial samples. To the best of our knowledge, this is the first work on the adversarial attack against hierarchical GNN-based graph classification models.

\end{abstract}


\section{Introduction}
%
%
%
%
\IEEEPARstart{I}{n} recent years, deep convolutional neural network has shown its outstanding performance in a variety of machine learning tasks. For example, the Convolutional Neural Network (CNN) has been widely used for image classification~\cite{he2016deep,szegedy2015going} and object detection~\cite{girshick2015fast}. Graph convolutional neural network (GCN), as a generalized CNN for the non-grid-like graph data, has also emerged as a powerful approach for the graph representation learning in many tasks, such as node classification~\cite{fey2018splinecnn,sun2019adagcn,wu2019simplifying,gao2018large} and edge prediction~\cite{pinter2018predicting,nathani2019learning}. There are mainly two categories of graph convolutional neural networks: spatial GCNs and spectral GCNs. The spectral GCNs take the adjacency matrix of a graph and the node feature vectors as input, and perform the convolutions based on graph Fourier Transform~\cite{kipf2016semi,henaff2015deep,defferrard2016convolutional}, while the spatial GCNs aggregates the node representations from its neighborhood~\cite{micheli2009neural,atwood2016diffusion,gilmer2017neural}.



Recently, hierarchical Graph Neural Networks (GNNs), as an advanced version of GCNs have been proposed for better capturing the hierarchical structure of graphs, and they have shown great advantage compared to the original GCNs for the graph-level learning tasks, such as graph classification and graph similarity analysis \cite{ma2019graph,ying2018hierarchical,lee2019self,gao2019graph,ma2019deep,ma2019survey}. For instance, in \cite{zhang2019hierarchical}, GNNs with hierarchical graph pooling is proposed and applied for the graph classification task, where the graph pooling operation adaptively selects a subset of nodes to form an induced subgraph, and a structure learning mechanism is introduced for preserving the integrity of graph topological information. In \cite{wang2019heterogeneous}, hierarchical GNNs are developed for the similarity learning between graphs for unknown malware detection.      

While the GNNs have made a great progress on the graph-related tasks, the evaluation of the robustness of these neural networks also becomes an important topic. Some recent works have started to study the vulnerability of GNNs in node classification tasks and have shown that, 
these models can be attacked by the adversarial samples \cite{zugner2019adversarial,dai2018adversarial,bojchevski2018adversarial,ma2019attacking,xu2019topology,wu2019adversarial,wang2019attacking,chang2020restricted}. That is to say, by introducing an unnoticeable perturbation from the original graph, we can easily fool the GNNs to make a wrong prediction result. For instance, Daniel Zugner et al. \cite{zugner2018adversarial} design a surrogate model to attack the GCN classification framework for a specific node. Their model conducts a greedy search to identify those attack candidate nodes, thereby generating the adversarial graph by perturbing the candidate nodes' features and the connections between the candidate node and the target node. After the attack, The GCN classifier will make a wrong classification to the specific target. Dai Hanjun et al. \cite{dai2018adversarial} provide another approach. They model the attack procedure as a Finite Horizon Markow Decision Process and build up a reinforcement learning framework to generate the adversarial samples for attacking and evaluating the robustness of a family of GCNs. 

Despite the current work has made some progress in the adversarial attack on GNNs, most of the attack tasks are on the node level (e.g. node classification). The whole graph classification is another valuable topic in various application domains, such as the protein pattern classification based on the amino acid networks, and the malware detection \cite{wang2019heterogeneous}, etc. Many recent studies \cite{du2019graph,gao2019graph,ying2018hierarchical} have started working on the graph classification task using GNNs, in particular the hierarchical GNNs which has been shown to be more powerful for graph classification tasks. However, few studies have been conducted on the robustness of these graph classification neural network models. 

In this paper, we focus on the adversarial attack on the hierarchical GCNs for graph classification tasks. Specifically, a new adversarial graph generating strategy with the surrogate model is proposed. We take the-state-of-art hierarchical graph pooling (HGP) models as instantiations of the hierarchical GCNs and build up an adversarial attack framework to explore the vulnerability of the HGP models to the adversarial samples. There are three main challenges in this topic. 
\begin{itemize}
    \item How to determine an effective attack target set of nodes and edges for the attacker? If we randomly perturb one or a few nodes/edges, the graph classification results may not change because such a perturbation may not influence or destroy the graph intrinsic structure that is discriminative for the graph classification task. 
    \item How to design the surrogate model to generate effective adversarial samples and fool the graph pooling neural networks? Since there are extensive matrix multiplication operations and non-linear components in the GNNs' loss function, especially in the pooling process, how to address or avoid the high-cost issue in the computation of the gradient of the entire loss with respect to each candidate entry to be added/deleted is also a big challenge in designing surrogate models.
    \item The robustness is always an important factor to evaluate the performance of the models. Under the scenario of adversarial attacks, how to improve the robustness of the hierarchical GNNs based graph classification models? 
\end{itemize}

To address the aforementioned challenges, we propose a surrogate model consisting of convolution and pooling operators. We use this surrogate model in generating the adversarial samples to attack the target HGP models. In the meantime, we conduct a robust train on the target models in order to improve the robustness of these hierarchical graph pooling neural networks. Our contributions can be summarized as follows: 

\begin{itemize}
    \item We propose an adversarial attack framework to evaluate the vulnerability of the hierarchical GCNs in graph classification tasks. To the best of our knowledge, this is the first work on the adversarial attack against hierarchical GCN based graph classification models. 
    \item We design a surrogate model that consists of convolutional and pooling operators to generate adversarial samples to fool the hierarchical GCN based graph classification models. Specifically, considering the fact that the hierarchical pooling operators in the hierarchical GCN classification models tend to preserve the nodes that are more important for the graph classification (i.e., nodes with higher scores computed in the pooling layer), we set the preserved nodes by the pooling operator as our attack targets, and then perturb the attack targets slightly to fool the pooling operator in hierarchical GCNs so that wrong nodes will be preserved. 
    \item We use the gradient-based method to identify the candidate edges to be added or deleted. To simplify the computation and reduce the cost of calculating the gradient of the whole network loss with respect to each candidate entry, we propose a new loss function based on the pooling score function, where we try to minimize the attack target node's pooling score by adding/deleting an edge connecting to the attack target node.
    \item We evaluate the adversarial samples generated from multiple datasets by our surrogate model and demonstrate its superior transferability in attacking the current state-of-art graph classification models. 
    \item We conduct the robust train and demonstrate that the retrained graph classification models are capable to defend against the attack from the adversarial samples.
\end{itemize}
The rest of this paper is organized as follows. Some related work will be discussed in the next section. The notations and preliminary knowledge are given in Section~\ref{sec_prelim}. Then we present the proposed framework in Section~ \ref{sec_model}. The experimental results and analysis are shown in Section~\ref{sec_exp}, followed by some additional evaluations in Section~\ref{sec_eval}. Finally, we conclude the paper in Section~\ref{sec_conclusion}. 



\section{Related Work}
Adversarial machine learning work can date back to 2004, when Dalvi et al. \cite{dalvi2004adversarial} and Lowd and Meek \cite{lowd2005adversarial} showed that the linear classifiers of spam filter could be easily fooled by some elaborated modified spam emails \cite{biggio2018wild}. Since then, extensive works on adversarial attack have been done to develop attacks against machine learning models \cite{nelson2008exploiting,rubinstein2009antidote}, or to evaluate the robustness of the machine learning models under the adversarial attack \cite{barreno2010security}  for developing strategies to defense the attack \cite{klinkenberg2000detecting,globerson2006nightmare}. Most of the existing works in this area focus on image, speech and language domains, while related studies on graphs are currently still at an early stage \cite{sun2018adversarial}.

Based on the attack tasks, the existing adversarial attacks on graphs can be categorized as: node relevant attacks, edge relevant attacks, and graph relevant attack. Node relevant attacks usually generate adversarial samples by making an unnoticeable perturbation on the original graphs and these adversarial samples can disturb the node embedding \cite{chen2018fast,sun2018data} or node classification \cite{dai2018adversarial,zugner2018adversarial,xu2019topology,wu2019adversarial,wang2019attacking,chang2020restricted} process. Likely, some edge relevant attack works build up the attack algorithms to disturb the node embedding process and then disturb the link predictions \cite{chen2018link}. So far, very few work has been done on the graph level attack problem. In \cite{dai2018adversarial}, Hanjun Dai et al. first attempt to use reinforcement learning technique to address the attack problem on graph classification models. However, since their model is evaluated only on generated data, attacking the model trained on the synthetic graphs may not be generalizable to real-world graph data in various domains. And the efficiency of using reinforcement learning (RL) method is usually a problem too. 

There are mainly two types of strategies in the graph attacks: evasion attack, and poisoning attack. Evasion attack means that the parameters of trained model are fixed and the attacker tries to generate the adversarial samples from the trained fixed model. In \cite{dai2018adversarial}, evasion attacks are designed for both inductive learning setting and transductive learning setting, while most of other existing works \cite{chen2018fast,sun2018data,zugner2019adversarial} choose poisoning attacks, in which attacker tries to generate the adversarial samples on the training dataset. In this paper, we choose evasion attacks as our strategy.

\section{Preliminaries}
\label{sec_prelim}
\subsection{Graph Notation}
We consider the supervised graph classification problem on attributed graphs with different number of nodes. Each graph has a class label and each graph's node has a feature vector. Formally, let $G=(A,H)$ be any of attributed graph with $N$ nodes, where $A\in\{0,1\}^{N \times N}$ is the adjacency matrix and $H \in \mathcal{R}^{N \times D}$ represents the $N$ nodes' $D$ dimension features. ($H \in \{0,1\}^{N \times D}$ if the features are binarized).Since the node order will not affect the classification and attack work, therefore, $w.l.o.g.$, we assign the node-ids as: $i \in \{1,2,...,N\}$. Then let $h _{i}\in\{0,1\}^{1 \times D}$ be the node $i$'s feature vector and $L \in\{1,2,...,k\}$ be each graph's label, where $k$ is the number of classes.   
\subsection{Hierarchical Graph Pooling (HGP) Model}
HGP models are one kind of the state-of-art techniques to tackle the graph classification problem \cite{ma2019graph,zhang2019hierarchical}. Given an attribute graph $G=(A,H)$, and its corresponding class label $L$,
the goal of graph classification is to learn a function $f: G \rightarrow L$, which can map the input graph $G$ to the output label $L$. Noted that our work is set on the inductive learning scenario, which means the testing data never share any information in the training process.

The HGP models are usually comprised by convolution and pooling operations. The convolution operation aims to project the node features into a new space in which the node information and the relationships among nodes can be well preserved. The pooling operation aims to preserve $M (M<N)$ nodes who support and encode the entire graph structure. For the discarded nodes, their information will be aggregated into the preserved nodes. The pooling operation effectively solves the \textit{flatness} of the previous GCN which only propagate the information through the nodes and edges. Instead, the pooling method can aggregate the graph information in a hierarchical way \cite{ying2018hierarchical}.

The convolution operation in the graph pooling models always follow the previous GCN model \cite{kipf2016semi}. The only difference is that the non-linear activation function is used after pooling. Here, the output of the ($l+1$)-th convolutional layer is defined as:
\begin{equation}
H_{l+1}^{cov}=\Tilde{D}^{-\frac{1}{2}}_{l} \Tilde{A}^{pool}_{l} \Tilde{D}^{-\frac{1}{2}}_{l}H_{l}^{pool} W_{l},
\end{equation}

where $\Tilde{A}^{pool}_{l}=A^{pool}_{l} + I_{N}$ ($I_{N}$ is the self loop matrix). $\Tilde{D}_{l}$ is the degree matrix corresponding to $\Tilde{A}^{pool}_{l}$ and $H_{l}^{pool}$ is the node features propagated from the $l$-th pooling layer. Noted that the adjacency matrix $A$ changes in different convolutional layer since the pooling layer discards some nodes. $W_{l}$ is the trainable weight parameters of layer $l$. \\
\indent As for the pooling layer, different studies have proposed different strategies to identify nodes to be preserved\cite{lee2019self,zhang2019hierarchical}. Most of these studies define a score function to rank nodes and then preserve the nodes with high score. Here, we define a unified pooling function $P$ as the pooling operator. Formally, the ($l+1$)-th pooling layer is:
\begin{eqnarray}
A_{l+1}^{pool}, H_{l+1}^{pool}&=&P(A_{l}^{pool} H_{l+1}^{cov}, \theta_{l+1})\\
H_{l+1}^{pool}&=&\sigma({H_{l+1}^{pool}})
\end{eqnarray}
where $A_{l}^{pool}$ is the adjacency matrix generated by the $l$-th pooling layer.  $H_{l+1}^{cov}$ is the feature vector generated by the ($l+1$)-th convolution layer. $\theta_{l+1}$ is the parameters of the ($l+1$)-th pooling layer and $\sigma(\cdot)$ is the non-linear activation function (e.g. $ReLU$). \\
\indent At the end of HGP model, there are a couple of linear output layers. Here, we define a linear function $g$ to represent the linear layers and formalize the output $y$ of the HGP model as:
\begin{eqnarray}
y = softmax(g (H_{l+1}^{pool}, V^{linear}))
\end{eqnarray}
where $V^{linear}$ is the parameters of linear layers and the output $y$ is a probability map for each of the classes. Let $U=\{W,\theta,V^{linear}\}$ be the parameter set, the whole model loss function can be defined using the negative log likelihood:
\begin{eqnarray}
\mathcal{L}_{model}(U,A,X) = -\log(y[L]),
\end{eqnarray}
where $[\cdot]$ is the index operation and $L$ is the label. 
\subsection{Adversarial Attacks on Graph}
Generally speaking, adversarial attack problem is a bilevel optimization problem which includes the upper-level and lower-level loss function. Let $\hat{G}=(\hat{A},\hat{H})$ be the adversarial samples generated by slightly perturbing the original graph $G$. For the evasion attack, the parameters $U$ needs to be optimized to minimize the whole model loss, $\mathcal{L}_{model}$, in the model training stage. After training, the model is fixed and then the adversarial samples should maximize the attack loss, $\mathcal{L}_{attack}$ given the fixed model. Formally, the graph adversarial evasion attack problem can be depicted as:
\begin{eqnarray}
&\underset{\hat{G}\in\delta(G)}{\mathrm{max}}\sum_{j}{}{\mathcal{L}_{attack}(f(\hat{G}_{j},U^{*}),L_{j})} \nonumber \\
\text{s.t.} & U^{*}=\underset{U}{\mathop{\arg\min}} \sum_{j}{}{\mathcal{L}_{model}(f(G_{j},U),L_{j})}, 
\end{eqnarray}
where $j=1,2...$ is the graph-ids, and $\delta(\cdot)$ defines perturbation restrictions.

\section{Methodology}
\label{sec_model}
Comparing to the traditional GCN model containing only convolution operations, the $\mathcal{L}_{model}$ of the HGP model is more complicated because of nonlinear pooling operations (i.e. network nodes cut operation and score function integration operation, see more details in the Section 4.2) in the pooling layer. Since the $\mathcal{L}_{attack}$ in most of existed studies is based on or directly equals to the $\mathcal{L}_{model}$, therefore, it's difficult to optimize the upper-level function in Eq.$(6)$ by computing the gradient of $\mathcal{L}_{attack}$ to graph ($G$) due to the complexity and nonlinearity of the $\mathcal{L}_{attack}$. In this work, we propose a new solution to tackle this issue. The following section is organized into four parts. First part describes the workflow about the model attack and how to help the model in defending against the attack. Second part describes a new surrogate model. Third part explains how to generate the adversarial samples by attacking the pooling operation in the surrogate model using a gradient-based method. And the last part illustrates the details of gradient computation.

\begin{figure}[h]
\small
    \vspace{-0.0001in}
    \centering
    \raggedright
    \begin{minipage}[h]{0.45\linewidth}
    \centering
    \includegraphics[width=2.2\linewidth]{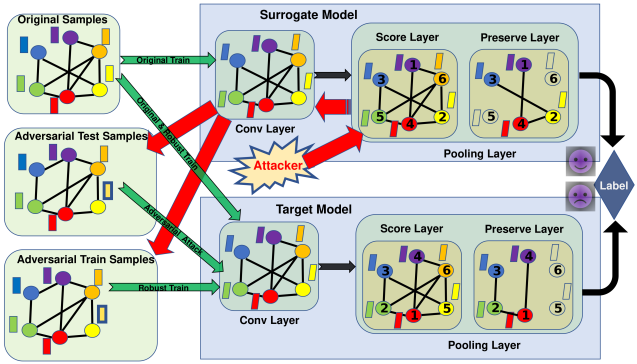}
    \end{minipage}
    \caption{Adversarial Attack on the HGP Model. Red arrow represents the adversarial samples' generating. Green arrow is the workflow and black arrow is the neural network propagation}
    \label{fig:my_label}
    \vspace{-0.0001in}
\end{figure}
\subsection{Workflow of Adversarial Attack}
We propose an evasion attack to the HGP model setting on the inductive-learning scenario (\textbf{Figure 1}). The workflow of our study can be summarized as follows: 

\textbf{(1) Surrogate Model Attack:} We build up a hierarchical pooling surrogate model and utilize it to generate the adversarial samples. To be specific, we firstly conduct the \textbf{original train}: train and test the surrogate model by using original training and testing data. Based on the rule of inductive-learning, testing data will never share their information in the training process. After the surrogate model gets well-trained, the model parameters are fixed. Then, we conduct the \textbf{adversarial samples generating}: use the well-trained surrogate model and the test data to generate the adversarial test samples. Finally, we conduct the \textbf{surrogate model attack}: use the adversarial test samples to test the surrogate model and get the attack results on surrogate model. 

\textbf{(2) Target Model Attack:} We firstly conduct the \textbf{original train} on several existed state-of-arts HGP models (\textbf{target models}). Then we use the adversarial test samples to attack these target models to show the transferability of the adversarial test samples. A good attack on target model can manifest that the surrogate model have the ability to generate the transferable adversarial samples. \\   
\indent \textbf{(3) Defense Against Adversarial Samples:} We conduct the \textbf{robust train} on target models. The process includes using the surrogate model to generate the adversarial train samples, and then using these adversarial train samples to retrain the target models. We will show the improvement of the robustness of the target models in the experiment section.

\subsection{Surrogate Model}
The pooling layer in many existed models is not trainable.\cite{zhang2019hierarchical,ma2019graph} The score function, for scaling the graph nodes, is obtained by calculating some graph distance \cite{zhang2019hierarchical} or node projection \cite{gao2019graph}. Such settings will lead to problems when we calculate the gradient of $\mathcal{L}_{attack}$ to the graph adjacency matrix $A$ or to the node feature matrix $H$. However, Junhyun Lee et al. \cite{lee2019self} proposed a self-attention graph pooling model which can parameterize the pooling layer and potentially facilitate the gradient computation. Therefore, inspired by Lee's work, we propose a new surrogate model to realize the gradient computation.

Two important characters of the surrogate model are: (1) has the ability to classify the graph well; (2) can efficiently generate the adversarial samples. Therefore, we propose an HGP model with one convolution layer and one pooling layer to be the surrogate model. At the end of pooling layer, there are two linear fully connected layers to output graph classification results. Following the Section 3, we formalize our surrogate model as:
\begin{eqnarray}
&y(A_{0},H_{0})=softmax(sel(H_{1}) \odot tanh(sel(S)) V_{1} V_{2}) \\
&\indent H_{1}(A_{0},H_{0})= \Tilde{D}_{0}^{-\frac{1}{2}} \Tilde{A}_{0} \Tilde{D}_{0}^{-\frac{1}{2}} H_{0} W \\
&\indent S(A_{0},H_{1})=\Tilde{D}_{0}^{-\frac{1}{2}} \Tilde{A}_{0} \Tilde{D}_{0}^{-\frac{1}{2}} H_{1} \theta , 
\end{eqnarray}
where $S$ is the parameterized node score function. $W,\theta,$ and $\{V_{1},V_{2}\}\\ \in V^{linear}$ are the parameters in the convolution and score function and linear layers. $\odot$ is the line-wise scalar multiplication. And $sel(\cdot)$ is the select function to preserve the selected nodes and discard others. i.e. If $node_{i}$ is selected to be discarded, then $sel(H)$ will set $H_{i,:}$ as 0. \\
\indent For a well-trained surrogate model, the input graph data $G_{0}=(A_{0},H_{0})$ propagates through the convolution layer and node features are encoded to a new feature space as $H_{1}$. Then $H_{1}$ and $A_{0}$ propagate into the pooling layer. In the pooling layer, $H_{1}$ and $A_{0}$ first participate to compute the node score function $S$. Then, based on the node score, the select function $sel(\cdot)$ preserves the top K nodes with high scores. Finally, the pooled feature vector will pass through the linear layers to generate a graph classification probability. The loss function of the model can be defined as:
\begin{eqnarray}
\mathcal{L}_{model}(U,A_{0},H_{0}) = -\log(y(A_{0},H_{0})[L]),
\end{eqnarray}
where $U=\{W,\theta,V^{linear}\}$, $L$ and $[\cdot]$ follow the previous definitions. 

\subsection{Generating Adversarial Samples}
We propose a gradient-based method to generate the adversarial samples. The foremost problem to generate the adversarial samples via a gradient based method is to define the attack loss function $\mathcal{L}_{attack}$. From Eqs $(7)\sim(10)$, we find that the $\mathcal{L}_{model}$ is a complicated nonlinear function. Therefore, if we directly use $\mathcal{L}_{model}$ as the attack loss, it will be very tedious to solve the upper-level function in Eq. $(6)$ via the gradient ascent method. Therefore, unlike the existed studies \cite{zugner2019adversarial,zugner2018adversarial}, we design a new $\mathcal{L}_{attack}$ which can avoid computing the gradient of nonlinear function. 
\subsubsection{\textbf{Attack Loss Design}}
The most significant component in the HGP model is the score function. A well-trained model can utilize the score function to select the preserved nodes which aggregate graph's hierarchical information. Hence, if we can generate the adversarial samples to decrease the scores of those preserved nodes lower than those of discarded nodes, the select function $sel(\cdot)$ will preserve wrong nodes and capture the wrong hierarchical information from the graph, which eventually lead to wrong prediction outputs. Therefore, we formalize our $\mathcal{L}_{attack}$ as:
\begin{eqnarray}
\mathcal{L}_{attack} = S(sel(A_{0}),sel(H_{1}))-S(sel(\hat{A_{0}}),sel(\hat{H_{0}})), 
\end{eqnarray}
and the adversarial sample generating process for each graph can be formalized as:

\begin{eqnarray}
&\underset{\hat{A_{0}}\in\delta(A_{0}),\hat{H_{0}}\in\delta(H_{0})}{\mathrm{max}}\mathcal{L}_{attack}
\end{eqnarray}


where $\hat{A_{0}}$ and $\hat{H_{0}}$ is the adjacency matrix and node feature vectors of adversarial samples. $\delta(\cdot)$ is an unnoticeable perturbation restriction when perturbing the original graph, which will be explained below.
\subsubsection{\textbf{Unnoticeable Perturbation Restriction}}
Since both the graph edges and node features can be attacked and the perturbations must be unnoticeable, we design the following restrictions for adjacency matrix and node feature respectively. \\
\textbf{(1) Restriction on Graph Edge Perturbation:} \\
Two restrictions are used to measure the graph edge perturbation. We use the edit-distance \cite{lovasz2012large} as the first measure and then set a budget $\triangle_{1}$ to restrict the edit-distance between $A_{0}$ and $\hat{A_{0}}$. Formally,
\begin{eqnarray}
&\delta_{1}(A_{0},\hat{A_{0}})=\frac{\mid A_{0}-\hat{A_{0}} \mid}{N^{2}} \leq \triangle_{1}
\end{eqnarray}
For the other restriction, we use the 2nd order DELTACON0 graph distance \cite{koutra2013deltacon} to capture the graph differences on the 2nd order level. Then we set another budget $\triangle_{2}$ to restrict the second order DELTACON0 graph distance between $A_{0}$ and $\hat{A_{0}}$. Formally, 
\begin{eqnarray}
\delta_{2}(A_{0},\hat{A_{0}})&=&\sqrt{(S_{0}-\hat{S}_{0})^{2}} \nonumber \\
&=&\sqrt{\sum_{i=1}^{N}\sum_{j=1}^{N}(\sqrt{s_{0,ij}}-\sqrt{\hat{s}_{0,ij}})^{2}}\leq \triangle_{2}  \nonumber \\
S_{0}&\approx& I + \epsilon A_{0} + \epsilon^{2} A_{0},  \nonumber \\
\hat{S}_{0}&\approx& I + \epsilon \hat{A}_{0} + \epsilon^{2} \hat{A}_{0}
\end{eqnarray}
\textbf{(2) Restriction on Node Feature Perturbation:} \\
\noindent An $l_{1}$ norm is used to measure the perturbation on the node feature vectors. We set up a budget $\triangle_{3}$ to restrict the perturbation between $H_{0}$ and $\hat{H}_{0}$. Formally, 
\begin{eqnarray}
&\delta_{3}(H_{0},\hat{H_{0}}) = \mid\mid H_{0}-\hat{H_{0}} \mid\mid_{1} \leq\triangle_{3}
\end{eqnarray}

\begin{algorithm}[h]
\small
\SetAlgoLined
\SetKwData{U^{*}}{U^{*}}\SetKwData{This}{this}\SetKwData{Up}{up}
\SetKwFunction{Union}{Union}\SetKwFunction{FindCompress}{FindCompress}
\SetKwInOut{Input}{input}\SetKwInOut{Output}{output}
\Input{$G=(A_0,H_0)$, restriction, $\triangle=\{\triangle_{1},\triangle_{2},\triangle_{3}\}$, graph label: $C_{L}$,$K$, $N$(number of Nodes)}
\Output{$\hat{G}=(\hat{A},\hat{H})$}
Train surrogate model: \\
$U^{*}$ $\leftarrow$ \indent $\underset{U}{\mathrm{min}}\mathcal{L}_{model}(U,A,H)$ 
\indent  \tcp*[f]{$U=\{W,\theta,V^{linear}\}$} \\
Fix model and generate adversarial samples \\
\While{$\delta(G,\hat{G})\leq\triangle$}{
\For{Top $K$ nodes}{
\For{$j=1,2,...,N$}{
Compute: $\nabla a_{i,j} = 
\frac{\partial (S(A_{0},H_{0})_{i}-S(A,H)_{i})}{\partial a_{ij}}$ \\
}
Find: $(i,j)$ $\leftarrow$ $\underset{(i,j)}{\mathrm{max}}\nabla a_{i,j}$\\
\If{$a_{i,j}=0\&\nabla a_{i,j}>0$ \textbf{OR} $a_{i,j}=1\&\nabla a_{i,j}<0$}{
Add/Delete $a_{i,j}$  \tcp*[f]{$\nabla a_{i,j}>0,Add;else,Delete$} \\
\lElse{\\
Delete $\nabla a_{i,j}$ and \textbf{Return} Find}
}
Compute: $\nabla h_{i}=\frac{\partial (S(A_{0},H_{0})_{i}-S(A,H)_{i})}{\partial h_{i}}$ \\
$\hat{h}_{i} \leftarrow h_{i} + \nabla h_{i}$
}
}
\caption{Generating Adversarial Samples}{\label{IR}}
\end{algorithm}


\subsubsection{\textbf{Gradient-Based Pooling Attack Algorithm}}
After the whole graph pooling model getting trained, all the model parameters are fixed. Then, we will generate the adversarial samples by our gradient-based pooling attack algorithm. Noted that $a_{i,j}$ is the element in $A_{0}$ at $(i,j)$ and $h_{i}$ is the $i$-th line of $H$. For the selected top K nodes, we compute the gradient of $\mathcal{L}_{attack}$ to each node feature and edge. i.e. If $node_{i}$ is selected, we will compute $\nabla h_{i} (\mathcal{L}_{attack})$ and $\nabla a_{ij} (\mathcal{L}_{attack})$ for each $j\in\{1,2,..., N\}$. Then we select the largest $\mid\nabla a_{ij}(\mathcal{L}_{attack})\mid$ and add/delete the edge at $a_{i,j}$ (Add the edge if the gradient is positive, otherwise, delete the edge). Also, we will update the $h_{i}$ to $(h_{i}+\nabla h_{i} (\mathcal{L}_{attack}))$. \\
\indent Notice that there are some special cases when perturbing the edge. At the largest gradient position $(i,j)$, if we will delete an edge but $a_{i,j}=0$, or, if we will add an edge but $a_{i,j}=1$. If these special cases happen, we will search for the 2nd largest gradient position. Such a process will continue until the special cases do not exist. We will conduct the above process from the Top 1 node to the Top K node, unless the perturbation restriction breaks. The details of the algorithm is shown in $\textbf{Algorithm 1}$

\subsection{Computation of Attack Gradient}
 In this part, will show how to compute the gradient of $\mathcal{L}_{attack}$ to adjacency matrix $A$(self-looped) and node feature $H$. w.r.t. each selected node. $WLOG$, we assume that $node_{i}$ is selected by the select function and become our attack target. Then the problems can be defined as follows:
\begin{problem}
Let $G_{0}=(A_{0},H_{0})$ be the original graph, compute the \\
$\frac{\partial (S(A_{0},H_{0})_{i}-S(A,H)_{i})}{\partial a_{ij}}$ for $j=\{1,2,...,N\}$. Here $a_{ij}$ is the element in $A$ at $(i,j)$ and $S(\cdot)$ is defined in Eq.(9). 
\end{problem}
For the convenience, we first give some matrix definitions. Let $\overline{D}=\Tilde{D}^{-\frac{1}{2}}$ be the diagonal matrix with each element as $\Tilde{d_{ii}}^{-\frac{1}{2}}$. The assist degree matrix $\check{D}$ is defined as:

\begin{small}
\begin{eqnarray}
\check{D} = 
\left[
\begin{matrix}
 \overline{d}_{11}\\
 \overline{d}_{22}\\
 \vdots\\
 \overline{d}_{NN}\\
\end{matrix}
\right]
\left[
\begin{matrix}
 \overline{d}_{11}\\
 \overline{d}_{22}\\
 \vdots\\
 \overline{d}_{NN}\\
\end{matrix}
\right]^{T}= 
\left[
\begin{matrix}
 \overline{d}_{11}^{2} &\overline{d}_{11}\overline{d}_{22}&\cdots&\overline{d}_{11}\overline{d}_{NN}\\
 \overline{d}_{22}\overline{d}_{11}&\overline{d}_{22}^{2}&\cdots&\overline{d}_{22}\overline{d}_{NN}\\
 \vdots & \vdots & \ddots & \vdots \\
 \overline{d}_{NN}\overline{d}_{11}&\overline{d}_{NN}\overline{d}_{22}&\cdots&\overline{d}_{NN}^{2}\\
\end{matrix}
\right]
\end{eqnarray}
\end{small}

where $(\cdot)^{T}$ is the matrix transpose.

\begin{theorem}
For each $j=\{1,2,...,N\}$, The partial derivative in PROBLEM 1. is: \\
\begin{eqnarray}
\frac{\partial (S(A_{0},H_{0})_{i}-S(A,H)_{i})}{\partial a_{ij}} &=& -\theta^{T}[\overline{d}_{ii}\overline{d}_{jj}(KHW)^{T}+M], \nonumber\\ 
K&=&\check{D}_{line_{j}}\circ A_{line_{j}},  \nonumber \\
M&=&a_{ii}\overline{d}^{2}_{ii}(\overline{d}_{ii}\overline{d}_{jj}H_{line_{j}}W)^{T}  
\end{eqnarray}
where $H\in\mathcal{R}^{N\times D}$, $W\in\mathcal{R}^{D\times F}$ and $\theta \in\mathcal{R}^{F\times 1}$ follow the previous definition, $\circ$ is the element-wise product. 
\end{theorem}
\begin{proof}
In the partial derivative in PROBLEM 1, we find the first term $(S(A_{0},H_{0})_{i}-S(A,H)_{i})$ is a constant which is unrelated to the gradient computing. So we only need to check the second term. Show that:
\begin{eqnarray}
-S(A,H)_{i}=-[(\check{D}_{line_{i}}\circ A_{line_{i}})(A\circ \check{D}HW)]\theta
\end{eqnarray}
Then we can expand $-\frac{\partial(S(A,H)_{i})}{\partial a_{ij}}$ for $\forall j \in \{1,2,...,N\}$ as:
\begin{eqnarray}
-\frac{\partial(S(A,H)_{i})}{\partial a_{ij}}=
-(&\theta_{1}(\overline{d}_{ii}\overline{d}_{jj}[X_{1}]+[Z_{1}]) + \nonumber\\
&\theta_{2}(\overline{d}_{ii}\overline{d}_{jj}[X_{2}]+[Z_{2}]) + \nonumber\\
&\cdots\cdots \nonumber \\
&\theta_{F}(\overline{d}_{ii}\overline{d}_{jj}[X_{F}]+[Z_{F}])).
\end{eqnarray}
where,

\begin{small}
\begin{eqnarray}
\left[
\begin{matrix}
 X_{1} \\
 X_{1} \\
 \vdots \\
 X_{F} \\
\end{matrix}
\right]&=&(
\left[
\begin{matrix}
  a_{j1}\overline{d}_{jj}\overline{d}_{11},\cdots , a_{jj}\overline{d}_{jj}\overline{d}_{jj}, \cdots ,a_{jN}\overline{d}_{jj}\overline{d}_{NN}
\end{matrix}
\right]HW)^{T} \nonumber \\
&=&(\check{D}_{line_{j}} \circ A_{line_{j}}HW)^{T}
\end{eqnarray}
\end{small}

\begin{eqnarray}
\left[
\begin{matrix}
  Z_{1} \\
  Z_{2} \\
  \vdots \\
  Z_{F}
\end{matrix}
\right]&=&
a_{ii}\overline{d}_{ii}^{2}(\overline{d}_{ii}\overline{d}_{jj}
\left[
\begin{matrix}
  h_{i1} h_{i2} \cdots h_{iD}
\end{matrix}
\right]W)^{T} \nonumber \\
&=& a_{ii}\overline{d}_{ii}^{2}(\overline{d}_{ii}\overline{d}_{jj}H_{line_{j}}W)^{T},
\end{eqnarray}
Concluded from Eqs.(19)$\sim$(21), Eq.(17) is proved. 
\end{proof}
 
\begin{problem}
Let $G_{0}=(A_{0},H_{0})$ be the original graph, compute the \\
$\frac{\partial (S(A_{0},H_{0})_{i}-S(A,H)_{i})}{\partial h_{i}}$. Here $h_{i}$ is the $i$-th line of $H$ which represents the feature of $node_{i}$ and $S(\cdot)$ is defined in Eq.(9)
\end{problem}
For the convenience, we define another assist degree matrix $\check{D}^{2}=\check{D}\circ\check{D}$ as: \\
\begin{eqnarray}
\check{D}^{2}=
\left[
\begin{matrix}
 \overline{d}_{11}^{4} &\overline{d}_{11}^{2}\overline{d}_{22}^{2} &\cdots &\overline{d}_{11}^{2}\overline{d}_{NN}^{2} \\
 \overline{d}_{22}^{2}\overline{d}_{11}^{2} &\overline{d}_{22}^{4} &\cdots &\overline{d}_{22}^{2}\overline{d}_{NN}^{2} \\ 
 \vdots &\vdots &\ddots &\vdots \\
 \overline{d}_{NN}^{2}\overline{d}_{11}^{2} &\overline{d}_{NN}^{2}\overline{d}_{22}^{2} &\cdots &\overline{d}_{NN}^{4}\\
\end{matrix}
\right]
\end{eqnarray}

\begin{theorem}
The partial derivative in PROBLEM 2. is: \\

\begin{small}
\begin{eqnarray}
\frac{\partial (S(A_{0},H_{0})_{i}-S(A,H)_{i})}{\partial h_{i}}=-(\check{D}^{2}_{line_{i}}(A_{line_{i}}^{T}\circ A_{col_{i}})\odot W)\theta,
\end{eqnarray}
\end{small}

where $W\in\mathcal{R}^{D\times F}$, $\theta \in\mathcal{R}^{F\times 1}$ and $\circ$ follow the previous definition. 
\end{theorem}

\begin{proof}
In the partial derivative in PROBLEM 2, only the second term $-S(A,H)_{i}$ is related to the gradient computing. So we only need to check the second term and show that:

\begin{small}
\begin{eqnarray}
-\frac{\partial(S(A,H)_{i})}{\partial h_{i}}  
&=&(\left[
\begin{matrix}
 \overline{d}_{ii}^{2}\overline{d}_{11}^{2}  
 \cdots 
 \overline{d}_{ii}^{2}\overline{d}_{ii}^{2} 
 \cdots 
 \overline{d}_{ii}^{2}\overline{d}_{NN}^{2} 
\end{matrix}
\right]
\left[
\begin{matrix}
a_{i1}a_{1i} \\
\vdots \\
a_{ii}a_{ii} \\
\vdots \\
a_{iN}a_{Ni} \\
\end{matrix}
\right] \odot W) \theta \nonumber \\
&=&(\check{D}^{2}_{line_{i}}(A_{line_{i}}^{T}\circ A_{col_{i}}) \odot W)
\theta
\end{eqnarray}
\end{small}

Therefore, Eq.(23) is proved. 
\end{proof}

\section{Experiments}
\label{sec_exp}
This section is organized into six parts. (A) data description; (B) adversarial samples generating and surrogate model attack; (C) the transferability of adversarial samples; (D) power analysis of the attack; (E) the feature attack vs. the edge attack; and (F) the robustness and model defending.    

\subsection{Dataset}
Six graph datasets are selected from the well-known Benchmark Data Sets for Graph Kernels\cite{KKMMN2016}. \textbf{DD} dataset contains graphs of protein crystal structures. The graph label indicates if the protein is enzyme or not\cite{dobson2003distinguishing}. \textbf{Mutagenicity} dataset contains graphs of chemistry molecular structure. The graph label indicates the Mutagenicity of the molecular \cite{kazius2005derivation}. \textbf{ER\_MD}, \textbf{BZR} and \textbf{DHFR} contain graphs to represent the chemical bond type. \cite{kriege2012subgraph,sutherland2003spline}. And \textbf{AIDS} dataset contains biological graphs to represent the antiviral character of different biology compounds\cite{riesen2008iam}. Details of dataset are summarized in \textbf{Table 1}.

\begin{table}[h]
\centering
  \caption{Dataset Statistics}
  \label{tab:freq}
  \begin{tabular}{lccc}
    \toprule
    Dataset&\# of Graphs&\# of Classes&Avg. \# of Edges\\
    \midrule
    DD & 1178 & 2 & 715.66\\
    Mutagenicity & 4337 & 2 & 30.77\\
    ER\_MD & 446 & 2 & 234.85\\
    DHFR & 467 & 2 & 44.54 \\
    AIDS & 2000 & 2 & 16.2 \\
    BZR  & 405 & 2 & 38.36 \\
  \bottomrule
\end{tabular}
\end{table}

\subsection{Surrogate Model Attack}
\subsubsection{\textbf{Experiment Setting}} 
Since our attack model is set on an inductive learning scenario, we first split each dataset into 80\% training, 10\% validation and 10\% testing subsets. We conduct the attack experiments on our surrogate model in 3 steps. Firstly, we conduct the original train to show that our surrogate model can well classify the original samples in a hierarchical way. To be specific, we train the surrogate model on training data and evaluate the model on validation data after each epoch.
The best trained model is saved and tested on testing data. Secondly, we generate the adversarial test samples from the fixed well trained model based on our attack method. The perturbation restrictions are set as: $\triangle_{1}=0.05$ and $\triangle_{2}=0.25$ ($\epsilon=0.0001$) which will hold the similarity between original and adversarial samples larger than $80\%$. And K is set as 50\%. In this step, we only perturb 5\% nodes and show that the attack performance is quite remarkable. Thirdly, we attack the surrogate model on adversarial test samples.

\subsubsection{\textbf{Surrogate Model Attack}}
Results of the original train and adversarial attack (with 5\% nodes perturbation) on surrogate model are shown in \textbf{Table 2}. The results of the original train show that our proposed surrogate model achieves a good classification performance (comparable to \cite{lee2019self,zhang2019hierarchical}). The results of adversarial attack show that our surrogate model can efficiently generate powerful adversarial samples which can impair the model's classification ability by perturbing a very few edges.  Comparing to the results of original train, the classification accuracies are reduced by 35.49\%, 33.86\%, 9.46\%, 10.91\%, 27.69\%, 9.10\% for DD, Mutagenicity, ER\_MD, DHFR, AIDS and BZR data after attack.

\begin{table}[h]
\centering
  \caption{Original Train v.s. Adversarial Attack}
  \label{tab:freq}
  \begin{tabular}{lcc}
    \toprule
    Dataset & Original Train & Adversarial Attack \\
    \midrule    
    DD & 76.84\% & 49.57\% \\
    Mutagenicity & 74.02\% & 48.96\%  \\
    ER\_MD & 76.09\% & 68.89\%   \\
    DHFR & 71.43\% & 63.64\%   \\
    AIDS & 97.50\% & 70.50\%   \\
    BZR & 82.50\% & 75.00\% \\
   \bottomrule
\end{tabular}
\end{table}

\subsection{Transferability of the Adversarial Samples}
\subsubsection{\textbf{Experiment Setting}}
We adopt two recent state-of-art HGP models, HGP-SL \cite{zhang2019hierarchical} and SAG \cite{lee2019self} as our target models. The training and validation data are used to conduct the original train using these two models and the testing data are used to show the performances of original trains. Then, following the rule of evasion attack, we use the generated adversarial test samples to attack these two target models. Meanwhile, for each dataset, we set up a baseline method as randomly attack 5\% selected nodes. 

\begin{table}
\centering
\caption{Transferability of Adv. Samples on Target Models}
\setlength{\tabcolsep}{0.75mm}{
  \begin{tabular}{lcccccc}
    \hline
    \multicolumn{1}{c|}{\textbf{Dataset}} 
    & \multicolumn{3}{|c|}{\textbf{SAG}} 
    &\multicolumn{3}{|c}{\textbf{HGP-SL}} \\
    \hline
    \multicolumn{1}{c|}{\textbf{}} 
    &\multicolumn{1}{|c|}{\text{Orig. Train}} 
    &\multicolumn{1}{|c|}{\text{Attack}}
    &\multicolumn{1}{|c|}{\text{Baseline}}
    &\multicolumn{1}{|c|}{\text{Orig. Train}}
    &\multicolumn{1}{|c|}{\text{Attack}}  
    &\multicolumn{1}{|c}{\text{Baseline}}\\
    \hline
    \multicolumn{1}{l|}{\text{DD}} 
    &\multicolumn{1}{|c}{$75.95\%$}
    &\multicolumn{1}{c}{$55.42\%$}
    &\multicolumn{1}{c|}{$64.41\%$}    
    &\multicolumn{1}{|c}{$79.98\%$}
    &\multicolumn{1}{c}{$57.14\%$}
    &\multicolumn{1}{c}{$66.10\%$}\\
    \multicolumn{1}{l|}{\text{Mutagenicity}} 
    &\multicolumn{1}{|c}{$76.32\%$}
    &\multicolumn{1}{c}{$43.34\%$}
    &\multicolumn{1}{c|}{$66.35\%$}
    &\multicolumn{1}{|c}{$81.61\%$}
    &\multicolumn{1}{c}{$61.84\%$}
    &\multicolumn{1}{c}{$73.33\%$}\\ 
    \multicolumn{1}{l|}{\text{ER\_MD}} 
    &\multicolumn{1}{|c}{$70.56\%$}
    &\multicolumn{1}{c}{$47.80\%$}
    &\multicolumn{1}{c|}{$60.00\%$}    
    &\multicolumn{1}{|c}{$78.26\%$}
    &\multicolumn{1}{c}{$50.00\%$} 
    &\multicolumn{1}{c}{$62.22\%$}\\
    \multicolumn{1}{l|}{\text{DHFR}} 
    &\multicolumn{1}{|c}{$72.73\%$}
    &\multicolumn{1}{c}{$63.64\%$}
    &\multicolumn{1}{c|}{$68.09\%$}
    &\multicolumn{1}{|c}{$75.32\%$}
    &\multicolumn{1}{c}{$63.64\%$}
    &\multicolumn{1}{c}{$68.09\%$}\\
    \multicolumn{1}{l|}{\text{AIDS}} 
    &\multicolumn{1}{|c}{$95.50\%$}
    &\multicolumn{1}{c}{$89.00\%$}
    &\multicolumn{1}{c|}{$93\%$}
    &\multicolumn{1}{|c}{$98.50\%$}
    &\multicolumn{1}{c}{$78.50\%$}
    &\multicolumn{1}{c}{$92.50\%$}\\ 
    \multicolumn{1}{l|}{\text{BZR}} 
    &\multicolumn{1}{|c}{$87.50\%$}
    &\multicolumn{1}{c}{$85.00\%$}
    &\multicolumn{1}{c|}{$85.00\%$}
    &\multicolumn{1}{|c}{$90.00\%$}
    &\multicolumn{1}{c}{$82.50\%$}
    &\multicolumn{1}{c}{$87.50\%$}\\ 
    \bottomrule
\end{tabular}}
\end{table}

\subsubsection{\textbf{Transferability of Adversarial Samples}}
The results of the original train and adversarial attack on HGP-SL and SAG models are shown in \textbf{Table 3}. We reproduce the previous work and obtain a similar test accuracy in the original train. The attack results manifest that the adversarial samples are transferable to attack different HGP models. This is because the attack method successfully perturb the very few nodes which support the hierarchical structure of the graph. In other words, the adversarial samples represent different latent and intrinsic hierarchical structure information from the original samples although they look very similar. Thus, no matter what HGP model will capture different hierarchical structure information from adversarial or original data, thereby output a total different classification.

\subsection{Attack Power Analysis}
\subsubsection{\textbf{Experiment Setting}}Attack Power is quantified by the percentages of perturbed edges. We generate the adversarial test samples by perturbing different percentages of edges. Consider that too many edges perturbations will break the unnoticeable restriction, therefore we set the max perturbation on each graph as perturbing 25\% of edges. For ER\_MD, and DD dataset, we cannot reach 25\% edge perturbation because even if we attack all top K nodes, the perturbed edges percentage is still less than 25\%. For these two dataset, we set the max perturbation as perturbing all selected top K nodes. All the above processes are conducted under the unnoticeable restriction.

\begin{figure}[h]
    \vskip 0.001in
    \centering
    \subfigure[]{
    \centering
    \begin{minipage}[t]{0.9\linewidth}
    \centering
    \includegraphics[width=1.0\linewidth]{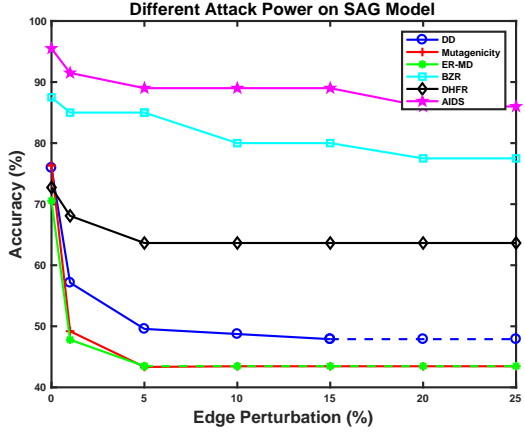}
    \end{minipage}
    }
    \subfigure[]{
    \centering
    \begin{minipage}[t]{0.9\linewidth}
    \centering
    \includegraphics[width=1\linewidth]{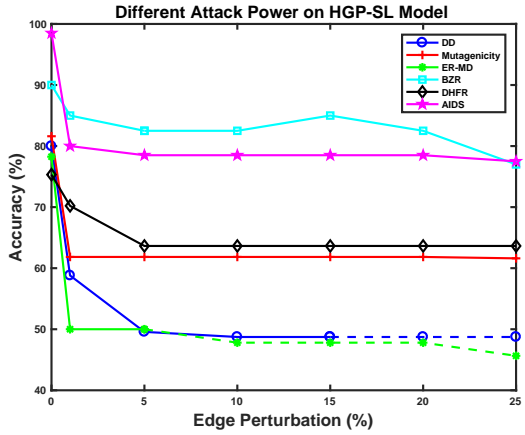}
    \end{minipage}
    }
    \caption{(a) shows the adversarial attack on SAG model under different attack strengths. (b) shows the adversarial attack on HGP-SL model under different attack strengths. Horizontal axis is edge perturbation percentage and vertical axis is classification accuracy}
    \label{fig:my_label}
    \vskip -0.001in
\end{figure}

\subsubsection{\textbf{Attack Power Analysis}}
The tendency of attack power is shown in the \textbf{Figure 2}. On the one hand, \textbf{Figure 2} shows that the adversarial attacks on both models become more powerful with the increase of the perturbation edges or more attack strength. On the other hand, it also shows that the attack power rapidly increases at the beginning of the attack when around 5\% edges are perturbed. Then, the attack power hardly increases whereas more edges are perturbed. This is because, at the beginning of the attack, the attack model perturbs very few important edges which hold the graph hierarchical structure. After these important nodes are attacked, other nodes are so trivial that have little contribution to holding the hierarchical structure of the graph. Thereby the perturbations on these trivial nodes will not enhance the attack consequences. Also, this results demonstrate that a small number of nodes contains the key information on the hierarchical structure of the whole graph. 

\subsection{Feature Attack v.s. Edge Attack}
\subsubsection{\textbf{Experiment Setting}}
Two dataset AIDS and BZR which include specific node features are selected to conduct this experiment. When generating the adversarial samples on node features, we set the unnoticeable restriction $\triangle_{3}=0.05$. We first conduct the feature attack on two target models using perturbed features and original edges. Then we combine the features attack and edges attack together.

\begin{figure}[h]
    \vskip 0.001in
    \raggedright
    \subfigure[]{
    \begin{minipage}[t]{1.0\linewidth}
    \centering
    \includegraphics[width=3.2in]{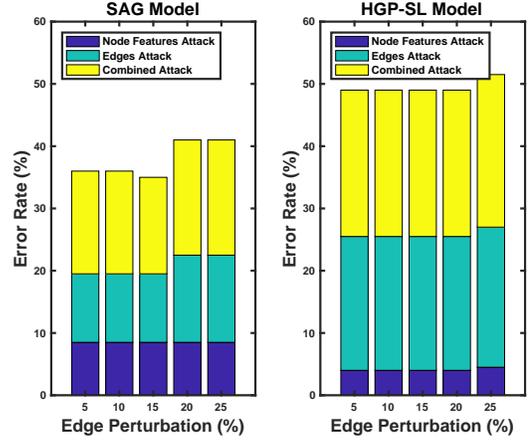}
    \end{minipage}
    }
    \subfigure[]{
    \begin{minipage}[t]{1.0\linewidth}
    \centering
    \includegraphics[width=3.2in]{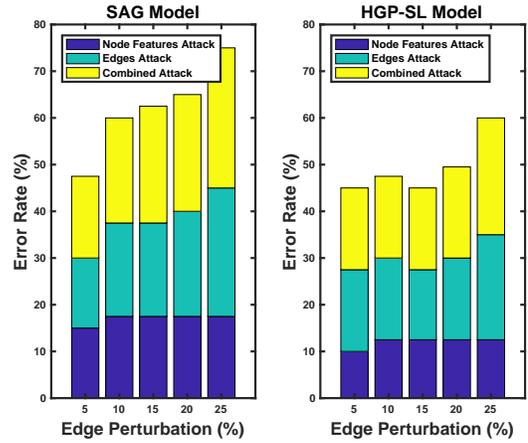}
    \end{minipage}
    }
    \caption{Node Feature Attack v.s. Edge Attack v.s. Combined Attack on AIDS (a) and BZR dataset (b). Large error rate represents powerful attack. Horizontal axis is edge perturbation percentage and vertical axis is classification error rate}
    \label{fig:my_label}
\end{figure}
\vskip -0.001in

\subsubsection{\textbf{Feature Attack v.s. Edges Attack}}
\textbf{Figure 3.} shows that both node feature attack and edges attack can reduce the performance of the target models and edges attack is more powerful than the node feature attack. This is because the perturbation on edge is discretely adding/deleting the edge, which will produce more errors than feature perturbations. The combined attack shows the most powerful attack on both target models across different dataset.

\subsection{Robust Train}
The goal of robust train is to improve the robustness of the target models to the adversarial samples with a small performance sacrifice.
\subsubsection{\textbf{Experiment Setting}}
 Firstly, we generate the adversarial training/validation samples. Then we mix the original and adversarial training data together to re-train the target models. The training process follows the original train in Section \textbf{5.2}. After the models are re-trained, we use adversarial test data to test the robustness of the re-trained model (\textbf{Robust Test}). 
We conduct the robust train based on DD and Mutagenicity dataset and retrain the model by using adversarial training data generated under different attack power.

\begin{figure}[h]
    \vskip 0.001in
    \raggedright
    \subfigure[]{
    \centering
    \begin{minipage}[t]{1.0\linewidth}
    \centering
    \includegraphics[width=3.2in]{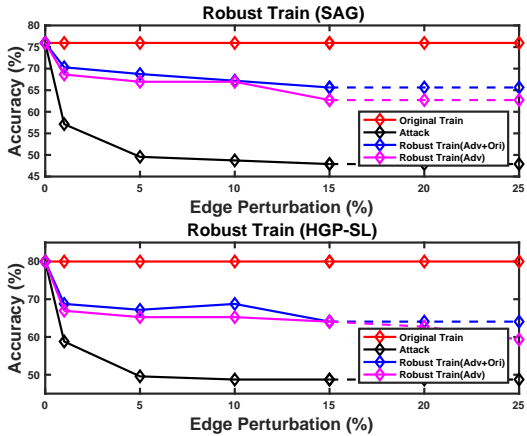}
    \end{minipage}
    }
    \subfigure[]{
    \centering
    \begin{minipage}[t]{1.0\linewidth}
    \centering
    \includegraphics[width=3.2in]{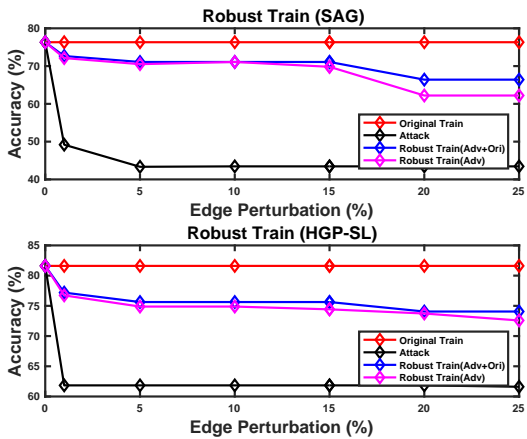}
    \end{minipage}
    }
    \caption{Show the robust train results on (a) DD dataset and (b) Mutagenicity dataset across two target models. Horizontal axis is edge perturbation percentage and vertical axis is classification accuracy}
    \label{fig:my_label}
    \vskip -0.001in
\end{figure}

\subsubsection{\textbf{Robust Train Analysis}} \textbf{Figure 4} shows the robust test on mixed (\textbf{Original+Adversarial}) samples as well as on \textbf{Adversarial} samples. The performance of the robust-trained models on the adversarial samples (pink line) shows a significant robustness improvement in compared with the performance of the original models on adversarial samples (black line). Moreover, the robust test on mixed samples (blue line) indicates the re-trained models has very few performance budget comparing to original models(red line).

\section{Evaluation}
\label{sec_eval}
The basic idea for our attack model is generating adversarial samples to attack the hierarchical structure of the original graphs by making an "unnoticeable" perturbation. Therefore, we firstly evaluate the perturbation between adversarial and original samples. Then, we show how the hierarchical structure is affected under the attack. To show the perturbation between original and adversarial samples, we compute the node degree distribution for each original and adversarial samples. Then we calculate the mean Kullback-Leibler (KL) divergence \cite{kullback1987letter} between each pairs of node degree distributions (\textbf{Table 4}). The lower value of mean KL divergence indicates the smaller difference between original and adversarial samples.

\begin{table}[h]
\centering
  \caption{Evaluation}
  \label{tab:freq}
  \begin{tabular}{lccc}
    \toprule
    Dataset& KL divergence(std.)& CL(p-value) & GRC(p-value)\\
    \midrule
    DD           & 7.60e-3(3.49e-2)   & 3.71e-6  & 1.77e-5 \\
    Mutagenicity & 5.11e-2(6.11e-2)  & 2.53e-9  & 1.50e-3 \\
    ER\_MD       & 1.05e-6(3.02e-9)& 1.11e-3  & 6.00e-4 \\
    DHFR         & 4.99e-2(2.10e-2)   & 4.41e-10 & 8.27e-20\\
    AIDS         & 4.64e-2(1.40e-1)  & 3.11e-3  & 7.01e-6 \\
    BZR          & 7.68e-2(2.09e-2)  & 4.95e-10 & 3.97e-8 \\
  \bottomrule
\end{tabular}
\end{table}

\begin{figure}[h]
    \small
    \vskip 0.001in
    \centering
    \subfigure[]{
    \begin{minipage}[t]{0.45\linewidth}
    \centering
    \includegraphics[width=1.8in]{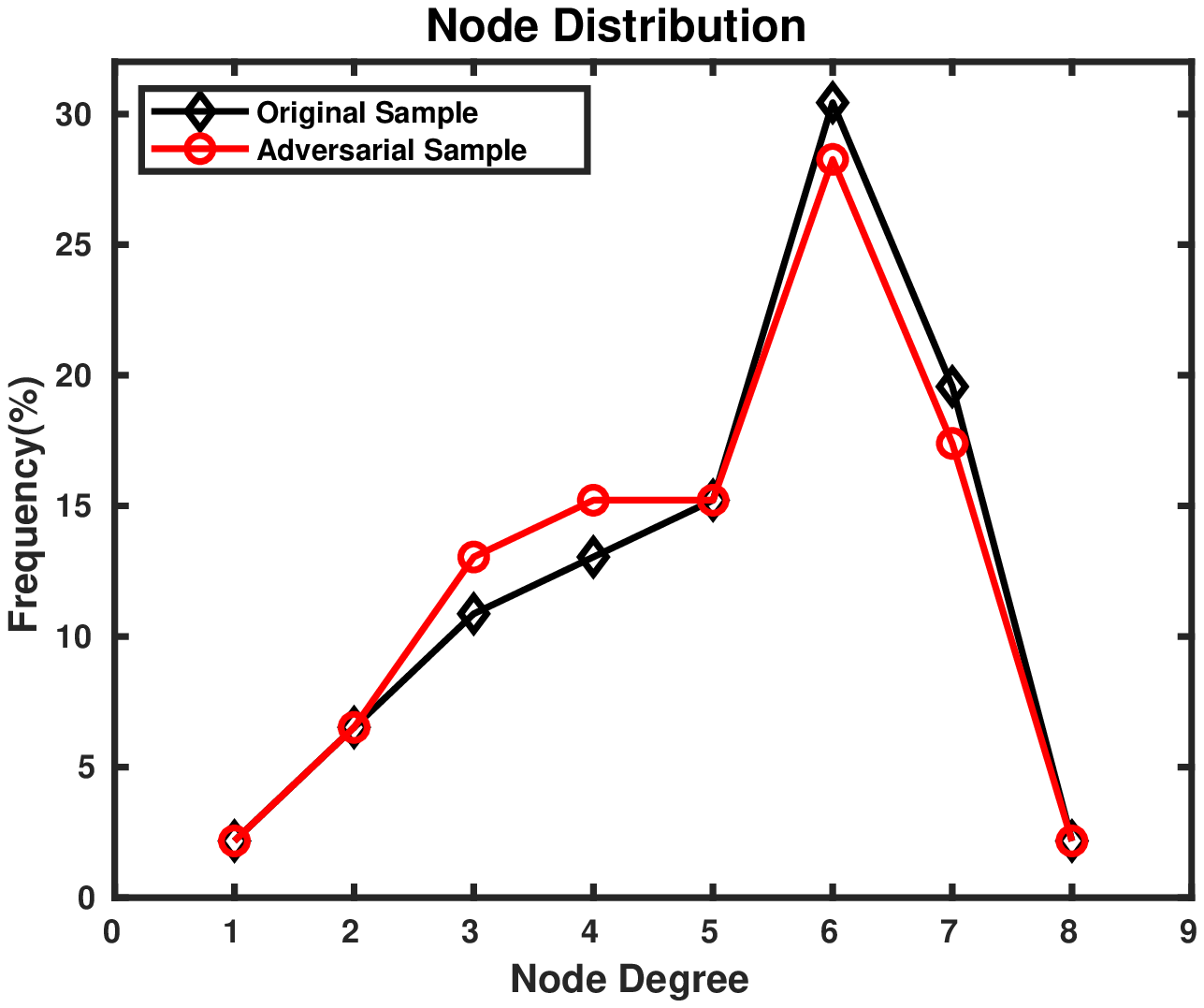}
    \end{minipage}
    }
    \subfigure[]{
    \centering
    \begin{minipage}[t]{0.45\linewidth}
    \centering
    \includegraphics[width=1.0\linewidth]{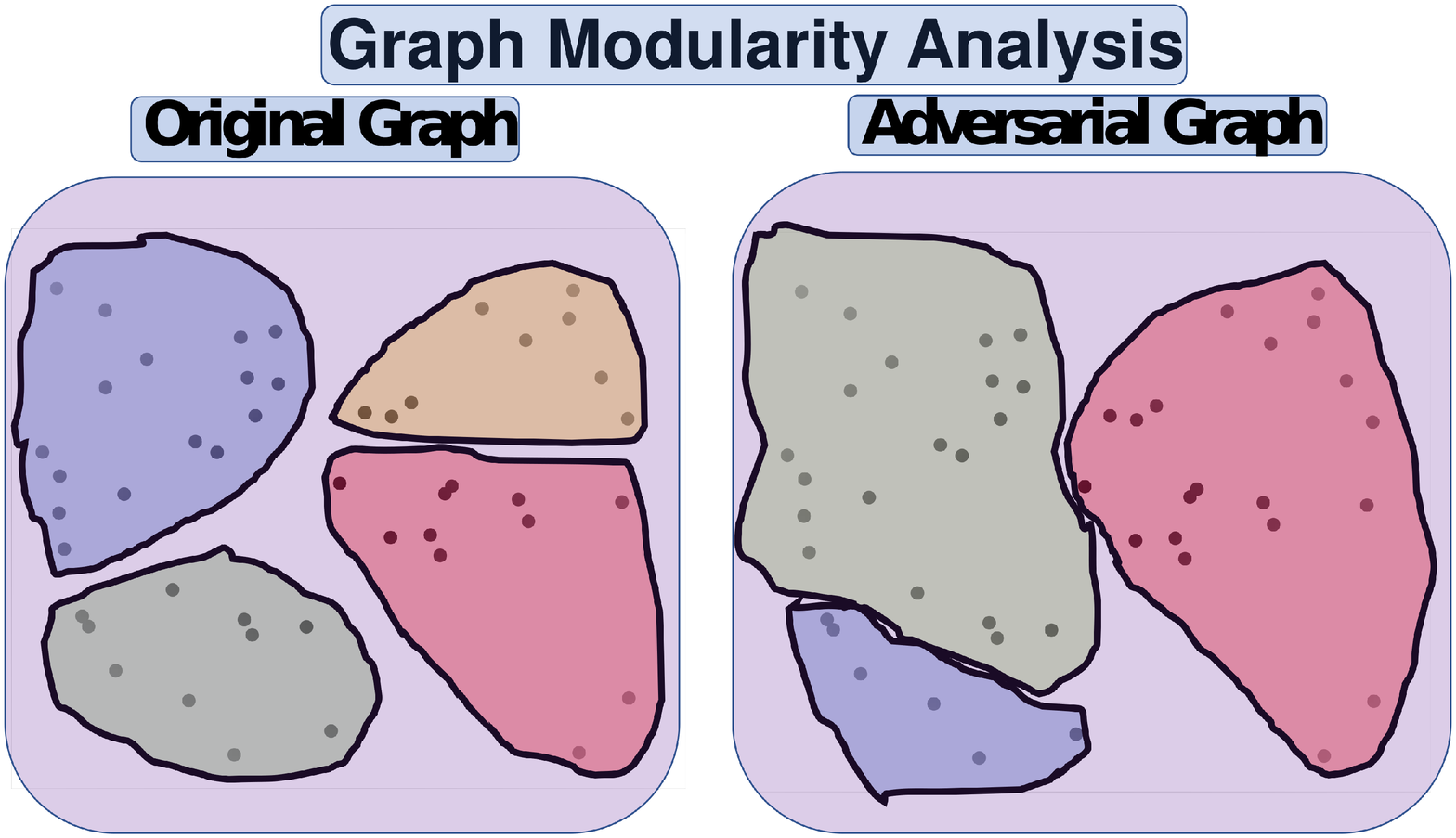}
    \end{minipage}
    }
    \caption{(a). Original sample's node distribution v.s. Adversarial sample's node distribution. horizontal coordinate is the node degree and vertical coordinate is frequency (b). Original sample's modularity v.s. Adversarial sample's modularity.}
    \label{fig:my_label}
    \vskip -0.001in
\end{figure}

\indent In order to evaluate the attack on the graph hierarchical structure, two popular measures, Global-Reaching-Centrality (GRC) \cite{mones2012hierarchy} and community-louvain (CL) \cite{blondel2008fast}, are extracted for each original and adversarial sample. Then we conduct a student T test for each measure between the original sample and adversarial sample. All the p-values (\textbf{Table 4}) are less than 0.05, indicating a significant difference in the hierarchical structure between the original and adversarial samples. \textbf{Figure 5} demonstrates an example (a DD data) of node degree distribution and hierarchical structure for both original and adversarial samples.

\section{Conclusion}
\label{sec_conclusion}
In this paper, we proposed a new adversarial attack framework to evaluate the robustness of the HGP models in graph classification tasks. Also, we conduct the robust train to help these graph pooling models to improve their robustness under the attack of the adversarial samples. However, the HGP model in our study is only one type of GCN-based methods. There are many other types of whole graph embedding methods whose robustness are still not be verified. One of our future work will focus on some more generalized adversarial attack strategies to evaluate the robustness of other graph classification models (such as graph kernel models, topology methods, etc.). Moreover, some of the questions in the task of graph adversarial attack are not well defined. For example, unlike the adversarial attack on image tasks, there is hardly any uniformed criterion to measure the "unnoticeable" perturbation between the original and perturbed graph. Therefore, it is also necessary to investigate some more convinced criterion to quantify 'unnoticeable perturbations. Lastly, the discrete optimization method should always be considered when generating the adversarial samples. How to build up a better optimization algorithm to deal with the discrete problem (e.g. continuous relaxation) will also be included in our future work.

\bibliographystyle{IEEEtran}



\end{document}